\documentclass[twoside,leqno,twocolumn]{article}
\usepackage{amsmath}
\usepackage{ltexpprt}
\usepackage{algorithm}
\usepackage{algorithmic}
\usepackage{txfonts}
\usepackage{graphicx}
\usepackage{subcaption}
\usepackage{comment}
\usepackage{url}
\usepackage{appendix}

\begin{document}

\renewcommand{\algorithmicrequire}{\textbf{Input:}}
\renewcommand{\algorithmicensure}{\textbf{Output:}}

\title{\Large Achieving Non-Discrimination in Data Release}
\author{Lu Zhang\thanks{Computer Science and Computer Engineering, University of Arkansas} \\
\and
Yongkai Wu\footnotemark[1] \\
\and
Xintao Wu\footnotemark[1]}
\date{}

\maketitle


\begin{abstract} \small\baselineskip=9pt 
Discrimination discovery and prevention/removal are increasingly important tasks in data mining. Discrimination discovery aims to unveil discriminatory practices on the protected attribute (e.g., gender) by analyzing the dataset of historical decision records, and discrimination prevention aims to remove discrimination by modifying the biased data before conducting predictive analysis. In this paper, we show that the key to discrimination discovery and prevention is to find the meaningful partitions that can be used to provide quantitative evidences for the judgment of discrimination. With the support of the causal graph, we present a graphical condition for identifying a meaningful partition. Based on that, we develop a simple criterion for the claim of non-discrimination, and propose discrimination removal algorithms which accurately remove discrimination while retaining good data utility. Experiments using real datasets show the effectiveness of our approaches.
\end{abstract}

\section{Introduction}\label{sec:intro}

Discrimination discovery and prevention/removal has been an active research area recently \cite{hajian2013methodology,kamiran2012data,ruggieri2010data,romei2014multidisciplinary,feldman2015certifying}. Discrimination discovery is the data mining problem of unveiling evidence of discriminatory practices by analyzing the dataset of historical decision records, and discrimination prevention aims to ensure non-discrimination by modifying the biased data before conducting predictive analysis (e.g., building classifiers). Discrimination refers to unjustified distinctions of individuals based on their membership in a certain group. Federal Laws and regulations (e.g., Fair Credit Reporting Act or Equal Credit Opportunity Act) prohibit discrimination on several grounds, such as gender, age, marital status, sexual orientation, race, religion or belief, and disability or illness, which are referred to as the \emph{protected attributes}. Different types of discrimination have been introduced, which can be generally categorized as direct and indirect discrimination \cite{feldman2015certifying}. Direct discrimination occurs when individuals are treated less favorably in comparable situations explicitly due to their membership in a protected group; indirect discrimination refers to an apparently neutral practice which results in an unfair treatment of a protected group \cite{hajian2013methodology,romei2014multidisciplinary}. 
In this paper, we focus on the problem of discrimination discovery and prevention on direct discrimination. In the following, we simply say discrimination for direct discrimination.


For a quantitative measurement of discrimination, a general legal principle is to measure the difference in the proportion of positive decisions between the protected group and non-protected group \cite{romei2014multidisciplinary}. Discovering and preventing discrimination is not trivial. As shown by \v{Z}liobait\.{e} et al. \cite{zliobaite2011handling}, simply considering the difference measured in the whole population fails to take into account the part of differences that are explainable by other attributes, and removing all the differences will result in reverse discrimination. They proposed to partition the data by conditioning on a certain attribute, and remove the difference within each produced subpopulation. Their method, although in the right direction, has two significant limitations: 1) it does not distinguish whether a partition is meaningful for measuring and removing discrimination; 2) it does not consider the situation where there exist multiple meaningful partitions. 



\setlength{\tabcolsep}{3.5pt}

\begin{table}[t]
\small
\centering
\caption{Summary statistics of Example 1.}
\begin{tabular}{c c c c c c c c c}
\cline{1-9}
test score & \multicolumn{4}{ c }{L} &  \multicolumn{4}{ c }{H} \\
gender & \multicolumn{2}{ c }{female} & \multicolumn{2}{ c }{male} & \multicolumn{2}{ c }{female} & \multicolumn{2}{ c }{male} \\
\cline{1-9}
major & CS & EE & CS & EE & CS & EE & CS & EE \\
No. applicants & 450&	150&	150&	450&	300&	100&	100&	300 \\
admission rate & 20\%&	40\%&	20\%&	40\%&	50\%&	70\%&	50\%&	70\% \\
& \multicolumn{2}{ c }{25\%} & \multicolumn{2}{ c }{35\%} & \multicolumn{2}{ c }{55\%} & \multicolumn{2}{ c }{65\%} \\
\cline{1-9}
\end{tabular}
\label{tab:toy1}
\end{table}

\begin{table}[t]
\small
\centering
\caption{Summary statistics of Example 2.}
\begin{tabular}{c c c c c c c c c}
\cline{1-9}
major & \multicolumn{4}{ c }{CS} &  \multicolumn{4}{ c }{EE} \\
gender & \multicolumn{2}{ c }{female} & \multicolumn{2}{ c }{male} & \multicolumn{2}{ c }{female} & \multicolumn{2}{ c }{male} \\
\cline{1-9}
test score & L & H & L & H & L & H & L & H \\
No. applicants & 450&	300&	150&	100&	600&	300&	200&	100 \\
admission rate & 30\%&	50\%&	36\%&	40\%&	40\%&	60\%&	45\%&	50\% \\
& \multicolumn{2}{ c }{38\%} & \multicolumn{2}{ c }{38\%} & \multicolumn{2}{ c }{47\%} & \multicolumn{2}{ c }{47\%} \\
\cline{1-9}
\end{tabular}
\label{tab:toy2}
\end{table}



Typically, given a population, a partition is determined by a subset of non-protected attributes and a subpopulation is specified by a value assignment to the attributes. In our study we find that, the key to answering whether a population contains discrimination is to finding the meaningful partitions that can be used to provide quantitative evidences for the judgment of discrimination. We have two observations. First, not all partitions are meaningful and using inappropriate partitions will result in misleading conclusions. For example, consider a toy model for a university admission system that contains four attributes: \texttt{gender}, \texttt{major}, \texttt{test\_score}, and \texttt{admission}, where \texttt{gender} is the protected attribute, and \texttt{admission} is the decision. We assume there is no correlation between \texttt{gender} and \texttt{test\_score}. The summary statistics of the admission rate is shown in Table \ref{tab:toy1}. It can be observed that the average admission rate is $37\%$ for females and $46\%$ for males. It is already known that the judgment of discrimination cannot be made simply based on the average admission rates in the whole population and further partitioning is needed. If we partition the data conditioning on \texttt{test\_score} as shown in the table, there exist significant differences (from either $35\%-25\%$ for L or from $65\%-55\%$ for H) between the admission rates of females and males in the two subpopulations. However, intuitively \texttt{test\_score} should not be used for partitioning the data alone as it is uncorrelated with the protected attribute. In fact, this result is indeed misleading, since if we carefully examine the admission rates for each \texttt{major} or each combination \{\texttt{major}, \texttt{test\_score}\}, it shows no bias in any of the subpopulations. Therefore, it would be groundless if if a plaintiff tries to file a lawsuit of discrimination against the university by demonstrating admission rate difference either in the whole population or based on the partitioning on \texttt{test\_score}. 

Second, when there are multiple meaningful partitions, examining one partition showing no bias does not guarantee no bias based on other partitions. Consider a different example on the same toy model shown in Table \ref{tab:toy2}. The average admission rate now becomes $43\%$ equally for both females and males. Further conditioning on \texttt{major} still shows that females and males have the same chance to be admitted in the two subpopulations. However, when partitioning the data based on the combination \{\texttt{major}, \texttt{test\_score}\}, significant differences ($\geq 5\%$) between the admission rates of females and males present. The difference among applicants applied to either a major with test scores of L is clear evidence of discrimination against females. The difference among applicants applied to either a major with test scores of H can be treated as reverse discrimination against males, or tokenism where some strong male applicants are purposefully rejected to refute a claim of discrimination against females. So, the data publisher cannot make a non-discrimination claim.

The above two examples show that, any quantitative evidence of discrimination must be measured under a meaningful partition. In addition, to ensure non-discrimination, we must show no bias for all meaningful partitions. In this paper, we make use of the causal graphs to identify meaningful partitions and develop discrimination discovery and removal algorithms. A causal graph is a probabilistic graph model widely used for causation representation, reasoning and inference \cite{spirtes2000causation}. 
Our main results are: 1) a graphical condition for identifying a meaningful partition, which is defined by a subset of attributes referred to as the \emph{block set}; 2) a simple criterion for the claim of non-discrimination; and 3) discrimination removal algorithms which achieve non-discrimination while maximizing the data utility. Our approaches can be used to find quantitative evidences of discrimination for plaintiffs, or to achieve a non-discrimination guarantee for data owners. The experiments using real datasets show that our proposed approaches are effective in discovering and removing discrimination. 



The rest of the paper is organized as follows. Section \ref{sec:md} models discrimination and presents the graphical condition using the causal graph. Section \ref{sec:ddp} establishes the non-discrimination criterion, develops the discrimination discovery mechanism, and proposes discrimination removal algorithms. Section \ref{sec:rnc} discusses a relaxed non-discrimination criterion. The experimental setup and results are discussed in Section \ref{sec:ee}. Section \ref{sec:rw} summarizes the related work, and Section \ref{sec:cc} concludes the paper.

\section{Modeling Discrimination Using Causal Graph}\label{sec:md}
Consider a dataset $\mathcal{D}$ which may contain discrimination against a certain protected group. Each individual in $\mathcal{D}$ is specified by a set of attributes $\mathbf{V}$, which includes the protected attribute (e.g., \texttt{gender}), the decision attribute (e.g., \texttt{admission}), and the non-protected attributes (e.g., \texttt{major}). Throughout this paper, we use an uppercase alphabet, e.g., $X$, to represent an attribute; a bold uppercase alphabet, e.g., $\mathbf{X}$, to represent a subset of attributes, e.g., \{\texttt{gender}, \texttt{major}\}. We use a lowercase alphabet, e.g., $x$, to represent a domain value of attribute $X$; a bold lowercase alphabet, e.g., $\mathbf{x}$, to represent a value assignment to $\mathbf{X}$. We denote the decision attribute by $E$, associated with domain values of positive decision $e^{+}$ and negative decision $e^{-}$; denote the protected attribute by $C$, associated with two domain values $c^{-}$ (e.g., female) and $c^{+}$ (e.g., male). 
 
For the measurement of discrimination, we use \emph{risk difference} \cite{romei2014multidisciplinary} to measure the difference in the the proportion of positive decisions between the protected group and non-protected group. Formally, by assuming $c^{-}$ is the protected group, risk difference is defined as $\Delta P|_{\mathbf{s}} = \Pr(e^{+}|c^{+},\mathbf{s})-\Pr(e^{+}|c^{-},\mathbf{s})$, where $\mathbf{s}$ denotes a specified subpopulation produced by a partition $\mathbf{S}$. 
We say that the protected group is treated less favorably within subpopulation $\mathbf{s}$ if $\Delta P|_{\mathbf{s}} \geq \tau$, where $\tau>0$ is a threshold for discrimination depending on law. For instance, the 1975 British legislation for sex discrimination sets $\tau = 0.05$, namely a $5\%$ difference. To avoid reverse discrimination, we do not specify which group is the protected group. Thus, we use $|\Delta P|_{\mathbf{s}}|$ to deal with both scenarios where either $c^{-}$ or $c^{+}$ is designated as the protected group.


A DAG $\mathcal{G}$ is represented by a set of nodes and a set of arcs. Each node represents an attribute in $\mathcal{D}$. 
Each arc, denoted by an arrow $\rightarrow$ in the graph, connects a pair of nodes where the node emanating the arrow is called the parent of the other node. The DAG is assumed to satisfy the Markov condition, i.e., each node $X$ is independent of all its non-descendants conditioning on its parents $\mathrm{Par}(X)$. 
Each node is associated with a conditional probability table (CPT) specified by $\Pr(X|\mathrm{Par}(X))$. The joint probability distribution can be computed using the factorization formula \cite{koller2009probabilistic}:
\begin{equation}\label{eq:ff}
\Pr(\mathbf{v}) = \prod_{X\in \mathbf{V}}\Pr(x|\mathrm{Par}(X)).
\end{equation}
Spirtes et al. \cite{spirtes2000causation} have shown that, when causal interpretations are given to the DAG, i.e., each node's parents are this node's direct causes, the DAG represents a correct causal structure of the underlying data. In particular, the causation among the attributes are encoded in the missing arcs in the DAG: if there is no arc between two nodes in $\mathcal{G}$, then it guarantees no direct causal effect between the two attributes in $\mathcal{D}$. The DAG with the  causal interpretation is called the \emph{causal graph}.


In this paper, we assume that we have a causal graph $\mathcal{G}$ that correctly represents the causal structure of the dataset. We also assume that the arc $C\rightarrow E$ is present in $\mathcal{G}$, since the absence of the arc represents a zero direct effect of $C$ on $E$. A causal DAG can be learned from data and domain knowledge. In the past decades, many algorithms have been proposed to learn causal DAGs, and they are proved to be quite successful \cite{spirtes2000causation,neapolitan2004learning,colombo2014order,kalisch2007estimating}. 
In our implementation, we use the original PC algorithm \cite{spirtes2000causation} to build the causal graph.

\subsection{Identifying Meaningful Partition}
Discrimination occurs due to difference decisions made explicitly based on the membership in the protected group. As stated in \cite{foster2004causation}, all discrimination claims require plaintiffs to demonstrate the existence of a causal connection between the decision and the protected attribute. Given a partition $\mathbf{S}$, for $\Delta P|_{\mathbf{s}}$ to capture the discriminatory effect and be considered as a quantitative evidence of discrimination, one needs to prove that this difference is directly caused by the protected attribute. 
In causal graphs, causal effects are carried by the paths that trace arrows pointing from the cause to the effect. Specially, the direct causal effect of $C$ on $E$, if exists, is carried by the direct arc from $C$ to $E$, i.e., $C \rightarrow E$ in the causal graph. In the following, we show that $\Delta P|_{\mathbf{s}}$ captures the direct causal effect of $C$ on $E$ if $\mathbf{S}$ satisfies certain requirements, which we refer to as the block set.

We use the graphical criterion $d$-separation. Consider a dataset $\mathcal{D}$ and its represented causal graph $\mathcal{G}$. Nodes $X$ and $Y$ are said to be $d$-separated by a set of attributes $\mathbf{Z}$ in $\mathcal{G}$, denoted by $(X\Perp Y \mid \mathbf{Z})_{\mathcal{G}}$, if following requirement is met:
\begin{Definition}[\textit{d}-Separation \cite{spirtes2000causation}]\label{def:d}
Nodes $X$ and $Y$ are $d$-separated by $\mathbf{Z}$ if and only if $\mathbf{Z}$ blocks all paths from $X$ to $Y$. A path $p$ is said to be blocked by a set of nodes $\mathbf{Z}$ if and only if
\begin{enumerate}
	\item $p$ contains a chain $i\rightarrow m\rightarrow j$ or a fork $i\leftarrow m \rightarrow j$ such that the middle node $m$ is in $\mathbf{Z}$, or
	\item $p$ contains an collider $i\rightarrow m\leftarrow j$ such that the middle node $m$ is not in $\mathbf{Z}$ and no descendant of $m$ is in $\mathbf{Z}$.
\end{enumerate}
\end{Definition}
Attributes $X$ and $Y$ are said to be conditionally independent given a set of attributes $\mathbf{Z}$, denoted by $(X\Perp Y \mid \mathbf{Z})_{\mathcal{D}}$, if $\Pr(x|y,\mathbf{z})=\Pr(x|\mathbf{z})$ holds for all values $x,y,\mathbf{z}$. With the Markov condition, the $d$-separation criterion in $\mathcal{G}$ and the conditional independence relations in $\mathcal{D}$ are connected such that, if we have $(X\Perp Y \mid \mathbf{Z})_{\mathcal{G}}$, then we must have $(X\Perp Y \mid \mathbf{Z})_{\mathcal{D}}$.

We make use of the above connection to identify the direct causal effect of $C$ on $E$. We construct a new DAG $\mathcal{G}'$ by deleting the arc $C\rightarrow E$ from $\mathcal{G}$ and keeping everything else unchanged. 
Thus, the possible difference between the causal relationships represented by $\mathcal{G}'$ and $\mathcal{G}$ lies merely in the presence of the direct causal effect of $C$ on $E$. We consider a node set $\mathbf{B}$ such that $(E\Perp C \mid \mathbf{B})_{\mathcal{G}'}$, and use $\mathbf{B}$ to examine the conditional independence relations in $\mathcal{D}$. If there is no direct causal effect of $C$ on $E$ in $\mathcal{G}$, we should also obtain $(E\Perp C \mid \mathbf{B})_{\mathcal{G}}$, which entails $(E\Perp C \mid \mathbf{B})_{\mathcal{D}}$, i.e., $\Pr(e^{+}|c^{+},\mathbf{b}) = \Pr(e^{+}|c^{-},\mathbf{b}) = \Pr(e^{+}|\mathbf{b})$ for each value assignment $\mathbf{b}$ of $\mathbf{B}$. However, if we observe $\Pr(e^{+}|c^{+},\mathbf{b})\neq \Pr(e^{+}|c^{-},\mathbf{b})$, the difference must be due to the existence of the direct causal effect of $C$ on $E$. Therefore, $\Delta P|_{\mathbf{b}}=\Pr(e^{+}|c^{+},\mathbf{b}) - \Pr(e^{+}|c^{-},\mathbf{b})$ can be used to measure the direct causal effect of $C$ on $E$. On the other hand, if a node set $\mathbf{S}$ does not satisfy $(E\Perp C \mid \mathbf{S})_{\mathcal{G}'}$, then this conditional dependence between $C$ and $E$ given $\mathbf{S}$ that is not caused by the direct causal effect will also exist in $\Delta P|_{\mathbf{s}}$. As a result, $\Delta P|_{\mathbf{s}}$ cannot accurately measure the direct causal effect.

It must be noted that, for $\Delta P|_{\mathbf{b}}$ to correctly measure the direct causal effect, set $\mathbf{B}$ cannot contain any descendant of $E$ even when it satisfies the requirement $(E\Perp C \mid \mathbf{B})_{\mathcal{G}'}$. This is because when conditioning on $E$'s descendants, part of the knowledge of $E$ is already given since the consequences caused by $E$ is known. 


Based on the above analysis, we measure the direct causal effect of $C$ on $E$ using $\Delta P|_{\mathbf{b}}$ where the node set $\mathbf{B}$ satisfies the following requirements: (1) $(C\Perp E \mid \mathbf{B})_{\mathcal{G}'}$ holds; (2) $\mathbf{B}$ contains none of $E$'s decedents. We call such set $\mathbf{B}$ a \emph{block set}. 
By treating the direct causal effect of $C$ one $E$ as the effect of direct discrimination, $\Delta P|_{\mathbf{b}}$ is considered as a correct measure for direct discrimination if and only if $\mathbf{B}$ is a block set. Thus, if $|\Delta P|_{\mathbf{b}}|\geq \tau$, it is a quantitative evidence of discrimination against either $c^{-}$ or $c^{+}$ for subpopulation $\mathbf{b}$. As can be seen, each block set $\mathbf{B}$ forms a meaningful partition on the dataset where the direct causal effect of $C$ on $E$ within each subpopulation $\mathbf{b}$ can be correctly measured by $\Delta P|_{\mathbf{b}}$. On the other hand, for any partition that is not defined by a block set, the measured differences, which either contain spurious influences or have been explained by the consequences of the decisions, cannot accurately measure the direct causal effect and hence cannot be used to prove discrimination or non-discrimination. Therefore, we give the following theorem. The proof follows the above analysis.



\begin{theorem}\label{thm:dce}
A node set $\mathbf{B}$ forms a meaningful partition for measuring discrimination if and only if $\mathbf{B}$ is a block set, i.e., $\mathbf{B}$ satisfies: (1) $(C\Perp E \mid \mathbf{B})_{\mathcal{G}'}$ holds; (2) $\mathbf{B}$ contains none of $E$'s decedents, where $\mathcal{G}'$ is the graph constructed by deleting arc $C\rightarrow E$ from $\mathcal{G}$. Discriminatory effect is considered to present for subpopulation $\mathbf{b}$ if $|\Delta P|_{\mathbf{b}}| \geq \tau$, where $\Delta P|_{\mathbf{b}} = \Pr(e^{+}|c^{+},\mathbf{b}) - \Pr(e^{+}|c^{-},\mathbf{b})$. 
\end{theorem}



\section{Discrimination Discovery and Prevention}\label{sec:ddp}

\subsection{Non-Discrimination Criterion}
Now we develop the criterion that ensures non-discrimination for a dataset. Based on Theorem \ref{thm:dce}, if each block set $\mathbf{B}$ is examined and ensures that $|\Delta P|_{\mathbf{b}}|< \tau$ holds for each subpopulation $\mathbf{b}$, we can guarantee that the dataset contains no discriminatory effect and is not liable for any claim of direct discrimination. Otherwise, there exists a subpopulation $\mathbf{b}$ of a block set such that $|\Delta P|_{\mathbf{b}}|\geq \tau$, which implies that subpopulation $\mathbf{b}$ suffers the risk of being accused of discrimination.
Therefore, we give the following theorem.

\begin{theorem}\label{thm:nd}
Non-discrimination is claimed for $\mathcal{D}$ if and only if inequality $|\Delta P|_{\mathbf{b}}| < \tau$ holds for each value assignment $\mathbf{b}$ of each block set $\mathbf{B}$.
\end{theorem}

\begin{figure}
\centering
\includegraphics[height=0.8in, width=1.7in]{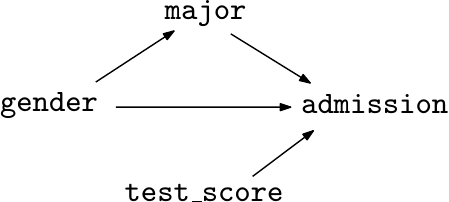}
\caption{Causal graph of an example university admission system.}
\label{fig:toy}
\end{figure}

We use the illustrative examples in Section \ref{sec:intro} to show how the criterion works. The causal graph of the examples is shown in Figure \ref{fig:toy}. There are two block sets in this graph: \{\texttt{major}\}, and \{\texttt{major},\texttt{test\_score}\}. Note that \texttt{test\_score} alone is not a block set. That is why conditioning on it will produce misleading results. For the example shown in Table \ref{tab:toy1}, examining both block sets shows no discriminatory effect. Thus, non-discrimination can be claimed. For the example shown in Table \ref{tab:toy2}, although examining \{\texttt{major}\} shows no discriminatory effect, when examining \{\texttt{major},\texttt{test\_score}\} we observe $|\Delta P|_{\textrm{\{math,B\}}}| = 0.06$, $|\Delta P|_{\textrm{\{math,A\}}}| = 0.10$, $|\Delta P|_{\textrm{\{biology,B\}}}| = 0.05$, and $|\Delta P|_{\textrm{\{biology,A\}}}| = 0.10$. Thus, the evidences of discrimination for four subpopulations are identified.


Although Theorem \ref{thm:nd} provides a clear criterion for non-discrimination, it requires examining each subpopulation of each block set. A brute force algorithm may have an exponential complexity. Instead of examining all block sets, the following theorem shows that we only need to examine one node set $\mathbf{Q}$, which is the set of all $E$'s parents except $C$, i.e., $\mathbf{Q}=\mathrm{Par}(E)\backslash \{C\}$. 

\begin{theorem}\label{thm:snd}
Non-discrimination is claimed if and only if inequality $|\Delta P|_{\mathbf{q}}| < \tau$ holds for each value assignment $\mathbf{q}$ of set $\mathbf{Q}$ where $\mathbf{Q}=\mathrm{Par}(E)\backslash \{C\}$.
\end{theorem}

\begin{proof}
We first give two lemmas and their proofs are given Appendices \ref{sec:thm:qd} and \ref{sec:thm:q}.

\begin{lemma}\label{thm:qd}
Given a value assignment $\mathbf{b}$ of a block set $\mathbf{B}$, $\mathbf{Q}=\mathrm{Par}(E)\backslash \{C\}$, we have
\begin{equation*}\label{eq:qd}
\Delta P|_{\mathbf{b}}=\sum_{\mathbf{Q}'}\Pr(\mathbf{q}'|\mathbf{b})\cdot \Delta P|_{\mathbf{q}},
\end{equation*}
where $\mathbf{q}'$ goes through all the possible combinations of the values of non-overlapping attributes $\mathbf{Q}' = \mathbf{Q} \backslash \mathbf{B}$.
For overlapping attributes $\mathbf{Q} \cap \mathbf{B}$, $\mathbf{q}$ and $\mathbf{b}$ have the same values.
\end{lemma}

\begin{lemma}\label{thm:q}
Node set $\mathbf{Q}$ of all $E$'s parents except $C$, i.e., $\mathbf{Q}=\mathrm{Par}(E)\backslash \{C\}$, must be a block set.
\end{lemma}

Lemma \ref{thm:qd} indicates that, for each value assignment $\mathbf{b}$ of each block set, $\Delta P|_{\mathbf{b}}$ can be expressed by a weighted average of $\Delta P|_{\mathbf{q}}$. If $|\Delta P|_{\mathbf{q}}| < \tau$ for each subpopulation of $\mathbf{Q}$, then it is guaranteed that $|\Delta P|_{\mathbf{b}}|<\tau$ holds for each subpopulation of each block set. According to Theorem \ref{thm:nd}, non-discrimination is claimed. Otherwise, Lemma \ref{thm:q} means that $|\Delta P|_{\mathbf{b}}|\geq \tau$ for at least one block set $\mathbf{Q}$, which provides the evidence of discrimination.
\end{proof}



\subsection{Discrimination Discovery}\label{sec:dda}
We present the non-discrimination certifying algorithm based on Theorem \ref{thm:snd}. The algorithm first finds set $\mathbf{Q}$ in the graph. Then, the algorithm computes $|\Delta P|_{\mathbf{q}}|=|\Pr(e^{+}|c^{+}, \mathbf{q})-\Pr(e^{+}|c^{-}, \mathbf{q})|$ for each subpopulation $\mathbf{q}$, and makes the judgment of non-discrimination based on the criterion. The procedure of the algorithm is shown in Algorithm \ref{alg:dd}.

\begin{algorithm}\small
\caption{Certifying of Non-Discrimination (Certify)}
\label{alg:dd}
\begin{algorithmic}[1]
\REQUIRE dataset $\mathcal{D}$, protected attribute $C$, decision $E$, user-defined parameter $\tau$
\ENSURE judgment of non-discrimination $judge$, parents of $E$ except $C$ $\mathbf{Q}$
		\STATE $\mathbf{Q}= findParent(E)\backslash \{C\}$
		\FORALL{value assignment $\mathbf{q}$ of $\mathbf{Q}$} 
			\STATE $|\Delta P|_{\mathbf{q}}|=|\Pr(e^{+}|c^{+}, \mathbf{q})-\Pr(e^{+}|c^{-}, \mathbf{q})|$
			\IF{$|\Delta P|_{\mathbf{q}}|\geq \tau$} 
				\RETURN $[false,\mathbf{Q}]$
			\ENDIF
		\ENDFOR
		\RETURN $[true,\mathbf{Q}]$
\end{algorithmic}
\end{algorithm}

The complexity from Line 2 to 8 is $O(|\mathbf{Q}|)$, where $|\mathbf{Q}|$ is the number of value assignments of $\mathbf{Q}$. The function $findParent(E)$ (Line 1) finds the parents of $E$ in a causal graph. A straightforward way is to first build a causal graph from the dataset using a structure learning algorithm (e.g., the classic PC algorithm), then find the parents of $E$ in the graph. 
The complexity of the PC algorithm
is bounded by the largest degree in the
undirected graph. In the worst case, the number of conditional
independence tests required by the algorithm is bounded by
$\frac{n^2{(n-1)}^{k-1}}{(k-1)!}$ where $k$ is the maximal degree of
any vertex and $n$ is the number of vertices.
However, in our algorithm we only need to identify the parents of $E$ without the need of building the complete network.
Thus, we can use local causal discovery algorithms such as the Markov blanket \cite{statnikov2013algorithms} to determine the local structure for the decision attribute $E$. 
We leave this part as our future work.

\subsection{Discrimination Removal}
When non-discrimination is not claimed, the discriminatory effects need to be removed by modifying the data before it is used for predictive analysis (e.g., building a discrimination-free classifier). Since the modification makes the data distorted, it may cause losses in data utility when compared with the original data. Thus, a general requirement in discrimination removal is to maximize the utility of the modified data while achieving non-discrimination. 
A naive approach such as used in \cite{feldman2015certifying} would be totally removing the protected attribute from the dataset to eliminates discrimination. However, as we shall show in the experiments, in this way the data utility would be greatly suffered. In this section, we propose two strategies that exactly remove discrimination while retaining good data utility. 

\subsubsection{Discrimination Removal by Modifying Causal Graph}
The first strategy modifies the constructed causal graph and uses it to generate a new dataset. Specifically, it modifies the CPT of $E$, i.e., $\Pr(e|c,\mathbf{q})$, to obtain a new CPT ${\Pr}'(e|c,\mathbf{q})$, to meet the non-discrimination criterion given by Theorem \ref{thm:snd}, i.e., $|{\Pr}'(e^{+}|c^{+},\mathbf{q})-{\Pr}'(e^{+}|c^{-},\mathbf{q})| < \tau$ for all subpopulations $\mathbf{q}$. The CPTs of all the other nodes are kept unchanged. The joint distribution of the causal graph after the modification can be calculated using the factorization formula (\ref{eq:ff}). After that, the algorithm generates a new dataset based on the modified joint distribution. Since the structure of the causal graph is not changed after the modification, $\mathbf{Q}$ is still the parent set of $E$ excluding $C$. Thus, according to Theorem \ref{thm:snd}, the newly generated dataset satisfies the non-discrimination criterion.

To achieve a good data utility, we minimize the difference between the original distribution (denoted by $P$) and the modified distribution (denoted by $P'$). We use the Euclidean distance, i.e., $d(P',P)=\sqrt{\sum_{\mathbf{V}}({\Pr}'(\mathbf{v})-\Pr(\mathbf{v}))^{2}}$, to measure the distance between the two distributions. We sort the nodes according to the topological ordering of the graph, and represent the sorted nodes as $\{C,\mathbf{X},E,\mathbf{Y}\}$. Note that we must have $\mathbf{Q}\subseteq \mathbf{X}$. Then, using the factorization formula (\ref{eq:ff}), $d(P',P)$ can be formulated as
\begin{equation*}
d(P',P) = \sqrt{\sum_{C,\mathbf{Q},E} \beta_{\mathbf{q}}^{c,e} \cdot \big( {\Pr}'(e|c,\mathbf{q})-\Pr(e|c,\mathbf{q}) \big)^{2}},
\end{equation*}
where $\beta_{\mathbf{q}}^{c,e} = \sum_{\mathbf{x}',\mathbf{y}}\big( \Pr(c)\Pr(\mathbf{x}|c)\Pr(\mathbf{y}|c,\mathbf{x},e) \big)^{2}$ and $\mathbf{X}'=\mathbf{X}\backslash \mathbf{Q}$. Thus, the optimal solution (denoted by $\Pr^{*}(e|c,\mathbf{q})$) that minimizes $d(P',P)$ can be obtained by solving the following quadratic programming problem.
\begin{equation*}
\begin{split}
\textrm{minimize} & \qquad \sum_{C,\mathbf{Q},E} \beta_{\mathbf{q}}^{c,e} \cdot \big( {\Pr}'(e|c,\mathbf{q})-\Pr(e|c,\mathbf{q}) \big)^{2} \\
\textrm{subject to} & \qquad \forall \mathbf{q},\quad |{\Pr}'(e^{+}|c^{+},\mathbf{q})-{\Pr}'(e^{+}|c^{-},\mathbf{q})| < \tau, \\
& \qquad \forall c,\mathbf{q},\quad {\Pr}'(e^{-}|c,\mathbf{q})+{\Pr}'(e^{+}|c,\mathbf{q}) = 1, \\
& \qquad \forall c,\mathbf{q},e,\quad {\Pr}'(e|c,\mathbf{q}) > 0.
\end{split}
\end{equation*}
The procedure of the algorithm is shown in Algorithm \ref{alg:mg}.

\begin{algorithm}[htb]\small
\caption{Removal by Modifying Graph (MGraph)}
\label{alg:mg}
\begin{algorithmic}[1]
\REQUIRE dataset $\mathcal{D}$, protected attribute $C$, decision $E$, user-defined parameter $\tau$
\ENSURE modeified dataset $\mathcal{D}^{*}$
		\STATE $[judge,\mathbf{Q}]=Certify(\mathcal{D},C,E,\tau)$
		\IF{$judge==true$} 
			\STATE $\mathcal{D}^{*}=\mathcal{D}$
		\ELSE
			\STATE Calculate the modified CPT of $E$: $\Pr^{*}(e|c,\mathbf{q})$
			\FORALL{$X\in \mathbf{V}\backslash \{E\}$} 
				\STATE $\Pr^{*}(x|\mathrm{Par}(X)) = \Pr(x|\mathrm{Par}(X))$
			\ENDFOR
			\STATE Calculate $P^{*}$ using Equation \eqref{eq:ff}
			\STATE Generate $\mathcal{D}^{*}$ based on $P^{*}$
		\ENDIF
		\RETURN $\mathcal{D}^{*}$
\end{algorithmic}
\end{algorithm}


The complexity of Algorithm \ref{alg:mg} depends on the complexity of building the causal graph and solving the quadratic programming. The complexity of building a causal graph has been discussed in Section \ref{sec:dda}. Note that since deriving the objective function needs information of the whole network, local causal discovery cannot be used to improve the algorithm. For the quadratic programming, it can be easily shown that, the coefficients of the quadratic terms in the objective function form a positive definite matrix. According to \cite{kozlov1980polynomial}, the quadratic programming can be solved in polynomial time. 

\subsubsection{Discrimination Removal by Modifying Dataset}
The second strategy directly modifies the decisions of selected tuples from the dataset to meet the non-discrimination criterion. For each value assignment $\mathbf{q}$, if $\Delta P|_{\mathbf{q}}\geq \tau$, we randomly select a number of tuples with $C=c^{-}$ and $E=e^{-}$, and change their $E$ values from $e^{-}$ to $e^{+}$. If $\Delta P|_{\mathbf{q}}\leq -\tau$, we select tuples similarly and change their $E$ values from $e^{+}$ to $e^{-}$. As result, we ensure that for each $\mathbf{q}$ we have $|\Delta P|_{\mathbf{q}}|\leq \tau$. 

For any $E$'s non-decedent $X$, according to the Markov condition, $X$ is independent of $E$ in each subpopulation specified by $E$'s parents, i.e., $C$ and $\mathbf{Q}$. Since the modified tuples are randomly selected in the subpopulation specified by $C$ and $\mathbf{Q}$, $X$ would still be independent of $E$ after the modification. Thus, all $E$'s non-decedents would be conditionally independent of $E$ given $C$ and $\mathbf{Q}$, implying that $\mathbf{Q}$ is still the parent set of $E$ excluding $C$ after the modification. According to Theorem \ref{thm:snd}, the modified dataset satisfies the non-discrimination criterion.


\setlength{\tabcolsep}{2.5pt}

\begin{table}[htb]
\scriptsize
\centering
\caption{Contingency table within subpopulation $\mathbf{q}$.}
\begin{tabular}{c c c c}
\cline{1-4}
& positive decision ($e^{+}$) & negative decision ($e^{-}$) & total \\
\cline{1-4}
protected group ($c^{-}$) & $n_{\mathbf{q}}^{c^{-}e^{+}}$ & $n_{\mathbf{q}}^{c^{-}e^{-}}$ & $n_{\mathbf{q}}^{c^{-}}$ \\
non-protected group ($c^{+}$) & $n_{\mathbf{q}}^{c^{+}e^{+}}$ & $n_{\mathbf{q}}^{c^{+}e^{-}}$ & $n_{\mathbf{q}}^{c^{+}}$ \\
\cline{1-4}
total & $n_{\mathbf{q}}^{e^{+}}$ & $n_{\mathbf{q}}^{e^{-}}$ & $n_{\mathbf{q}}$\\
\cline{1-4}
\end{tabular}
\label{tab:contingencyTable}
\end{table}

To calculate the number of tuples to be modified within each subpopulation $\mathbf{q}$, we express $\Delta P|_{\mathbf{q}}$ as $n_{\mathbf{q}}^{c^{+}e^{+}}/n_{\mathbf{q}}^{c^{+}}-n_{\mathbf{q}}^{c^{-}e^{+}}/n_{\mathbf{q}}^{c^{-}}$. Please refer to Table~\ref{tab:contingencyTable} for the meaning of the notations. For subpopulations with $\Delta P|_{\mathbf{q}}\geq \tau$, by selecting $\lceil n_{\mathbf{q}}^{c^{-}}\cdot (|\Delta P|_{\mathbf{q}}|-\tau) \rceil$ tuples with $C=c^{-}$ and $E=e^{-}$, and changing their $E$ values from $e^{-}$ to $e^{+}$, the value of $\Delta P|_{\mathbf{q}}$ would decrease by $\lceil n_{\mathbf{q}}^{c^{-}}\cdot (|\Delta P|_{\mathbf{q}}|-\tau)\rceil / n_{\mathbf{q}}^{c^{-}} \geq \Delta P|_{\mathbf{q}}-\tau$. Therefore, we have $\Delta P|_{\mathbf{q}} < \tau$ after the modification. The result is similar when $\Delta P|_{\mathbf{q}}\leq -\tau$. The pseudo-code of the algorithm is shown in Algorithm \ref{alg:md}.

\begin{algorithm}[htb]\small
\caption{Removal by Modifying Data (MData)}
\label{alg:md}
\begin{algorithmic}[1]
\REQUIRE dataset $\mathcal{D}$, protected attribute $C$, decision $E$, user-defined parameter $\tau$
\ENSURE modeified dataset $\mathcal{D}^{*}$
		\STATE $[judge,\mathbf{Q}]=Certify(\mathcal{D},C,E,\tau)$
		\IF{$judge==true$} 
			\STATE $\mathcal{D}^{*}=\mathcal{D}$
		\ELSE
			\FORALL{value assignment $\mathbf{q}$ of $\mathbf{Q}$} 
				\IF{$\Delta P|_{\mathbf{q}}>\tau$} 
					\STATE randomly select a set $\mathbf{T}$ of $\lceil n_{\mathbf{q}}^{c^{-}}\cdot (|\Delta P|_{\mathbf{q}}|-\tau) \rceil$ tuples with $C=c^{-}$ and $E=e^{-}$ in subpopulation $\mathbf{q}$, and change the values of $E$s from $e^{-}$ to $e^{+}$ to get the set $\mathbf{T}^{*}$ of the modified tuples
				\ELSIF{$\Delta P|_{\mathbf{q}}<-\tau$} 
					\STATE randomly select a set $\mathbf{T}$ of $\lceil n_{\mathbf{q}}^{c^{-}}\cdot (|\Delta P|_{\mathbf{q}}|-\tau) \rceil$ tuples with $C=c^{-}$ and $E=e^{+}$ in subpopulation $\mathbf{q}$, and change the values of $E$s from $e^{+}$ to $e^{-}$ to get the set $\mathbf{T}^{*}$ of the modified tuples
				\ENDIF
				\STATE $\mathcal{D}^{*}=\mathcal{D}^{*}\backslash \mathbf{T}\cup \mathbf{T}^{*}$
			\ENDFOR
		\ENDIF
		\RETURN $\mathcal{D}^{*}$
\end{algorithmic}
\end{algorithm}

The complexity of Algorithm \ref{alg:md} includes the complexity of finding $\mathbf{Q}$. Similar to Algorithm \ref{alg:dd}, we can identify $E$'s parents without building the whole network. Therefore, local discovery algorithms can be employed to improve the efficiency of algorithm. The complexity from Line 5 to 14 is bounded by the size of the original dataset, i.e., $O(|\mathcal{D}|)$.


\section{Relaxed Non-Discrimination Criterion}\label{sec:rnc}
So far, we treat dataset $\mathcal{D}$ as the whole population. In real situations, $\mathcal{D}$ may be a sample of the whole population, and $\Delta P|_{\mathbf{b}}$s under a block set $\mathbf{B}$ may vary from one subpopulation to another due to randomness in sampling, especially when the sample size is small. The $|\Delta P|_{\mathbf{b}}|$ values of a few $\mathbf{b}$ could be larger than $\tau$ due to the small sample size although the majority of $|\Delta P|_{\mathbf{b}}|$ values are smaller than $\tau$. In this situation, the dataset is claimed as containing discrimination based on the above criterion where all $|\Delta P|_{\mathbf{b}}|$s should be smaller than $\tau$ no matter of the majority of $\Delta P|_{\mathbf{b}}$ values.

In this section, we propose a relaxed $\alpha$-non-discrimination criterion which may perform better under the context of randomness and small samples by finding statistical evidences. Formally, for a given block set $\mathbf{B}$, we treat $\Delta P|_{\mathbf{B}}$ as a variable and treat the values of $\Delta P|_{\mathbf{b}}$s observed across all subpopulations as samples. We introduce a user-defined parameter, $\alpha$ ($0 < \alpha < 1$), to indicate a threshold for the probability of $|\Delta P|_{\mathbf{B}}| < \tau$. If $|\Pr(\Delta P|_{\mathbf{B}}| < \tau)\geq \alpha$, then we say no significant bias is observed under partition $\mathbf{B}$. If $\Pr(|\Delta P|_{\mathbf{B}}| < \tau)\geq \alpha$ holds for each block set, then $\alpha$-non-discrimination can be claimed for $\mathcal{D}$.

\begin{Definition}\label{def:rnd}
Given $\alpha$, $\alpha$-non-discrimination is claimed if $\Pr(|\Delta P|_{\mathbf{B}}| < \tau)\geq \alpha$ holds for each block set $\mathbf{B}$.
\end{Definition}
One challenge here is that we do not know the exact distribution of $\Delta P|_{\mathbf{B}}$ for estimating $\Pr(|\Delta P|_{\mathbf{B}}| < \tau)$ accurately. We propose to employ the Chebyshev's inequality \cite{anastassiou2009probabilistic}, which provides a lower bound of the probability for the value of a random variable lying within a given region, using its mean and variance. Note that the Chebyshev's inequality holds for any random variable irrespective of its distribution. The general form of the Chebyshev's inequality is given as follows.
\begin{theorem}[Chebyshev's inequality]
Let $X$ be a random variable with finite expected value $\mu$ and finite non-zero variance $\sigma^{2}$. Then for any real numbers $a < b$,
\begin{equation*}
\Pr(a<X<b) \geq 1-\frac{\sigma^{2}+(\mu-\frac{b+a}{2})^{2}}{(\frac{b-a}{2})^{2}}.
\end{equation*}
\end{theorem}

The following theorem shows a sufficient condition to satisfy Definition \ref{def:rnd} using the Chebyshev's inequality.

\begin{theorem}\label{thm:rnd}
Given $\alpha$, $\alpha$-non-discrimination is claimed if the following inequality holds for each block set $\mathbf{B}$:
\begin{equation*}
1-\frac{\sigma_{\mathbf{B}}^{2}+\mu_{\mathbf{B}}^{2}}{\tau^{2}} \geq \alpha,
\end{equation*}
where $\mu_{\mathbf{B}}$ and $\sigma_{\mathbf{B}}^{2}$ are mean and variance of $\Delta P|_{\mathbf{B}}$.
\end{theorem}
The proof is straightforward by replacing $X$ with $\Delta P|_{\mathbf{B}}$, $a$ with $-\tau$, and $b$ with $\tau$ in the Chebyshev's inequality.

We show Theorem \ref{thm:rnd} can be achieved by examining $\mathbf{Q}$ only.
\begin{theorem}\label{thm:srnd}
Given $\alpha$, $\alpha$-non-discrimination is claimed if the following inequalities holds for set $\mathbf{Q}$:
\begin{equation*}
1-\frac{\hat{\sigma}_{\mathbf{Q}}^{2}+\hat{\mu}_{\mathbf{Q}}^{2}}{\tau^{2}} \geq \alpha,
\end{equation*}
where $\hat{\mu}_{\mathbf{B}}=\sum_{\mathbf{B}}\Pr(\mathbf{b})\cdot \Delta P|_{\mathbf{b}}$ and $\hat{\sigma}_{\mathbf{B}}^{2}=\sum_{\mathbf{B}}\Pr(\mathbf{b})(\Delta P|_{\mathbf{b}}-\hat{\mu}_{\mathbf{B}})^{2}$.
\end{theorem}

\begin{proof}
The proof is straightforward by giving two lemmas:
\begin{lemma}\label{thm:mu}
For each block set $\mathbf{B}$, $\hat{\mu}_{\mathbf{B}}=\hat{\mu}_{\mathbf{Q}}$, where $\mathbf{Q}=\mathrm{Par}(E)\backslash \{C\}$.
\end{lemma}
\begin{lemma}\label{thm:sigma}
For each block set $\mathbf{B}$, $\hat{\sigma}_{\mathbf{B}}^{2}\leq \hat{\sigma}_{\mathbf{Q}}^{2}$, where $\mathbf{Q}=\mathrm{Par}(E)\backslash \{C\}$.
\end{lemma}
Refer to the Appendices \ref{sec:thm:mu} and \ref{sec:thm:sigma} for proof details.
\end{proof}

\section{Experiments}\label{sec:ee}
In this section, we conduct experiments for discrimination discovery and removal algorithms by using two real data sets: the Adult dataset \cite{adultdataset} and the Dutch Census of 2001 \cite{dutchdataset}, 
and compare our algorithms with the conditional discrimination removal methods proposed in \cite{zliobaite2011handling}.

The causal graphs are constructed by utilizing an open-source software TETRAD \cite{tetrad}, which is a widely used platform for causal modeling. We employ the original PC algorithm and set the significance threshold $0.01$ used for conditional independence testing in causal graph construction. The quadratic programming is solved using CVXOPT \cite{cvxopt}.
All experiments were conducted with a PC workstation with 16GB RAM and Intel Core i7-4770 CPU.

\subsection{Discrimination Discovery}

\begin{figure*}[ht]
	\centering
		\includegraphics[width=6in]{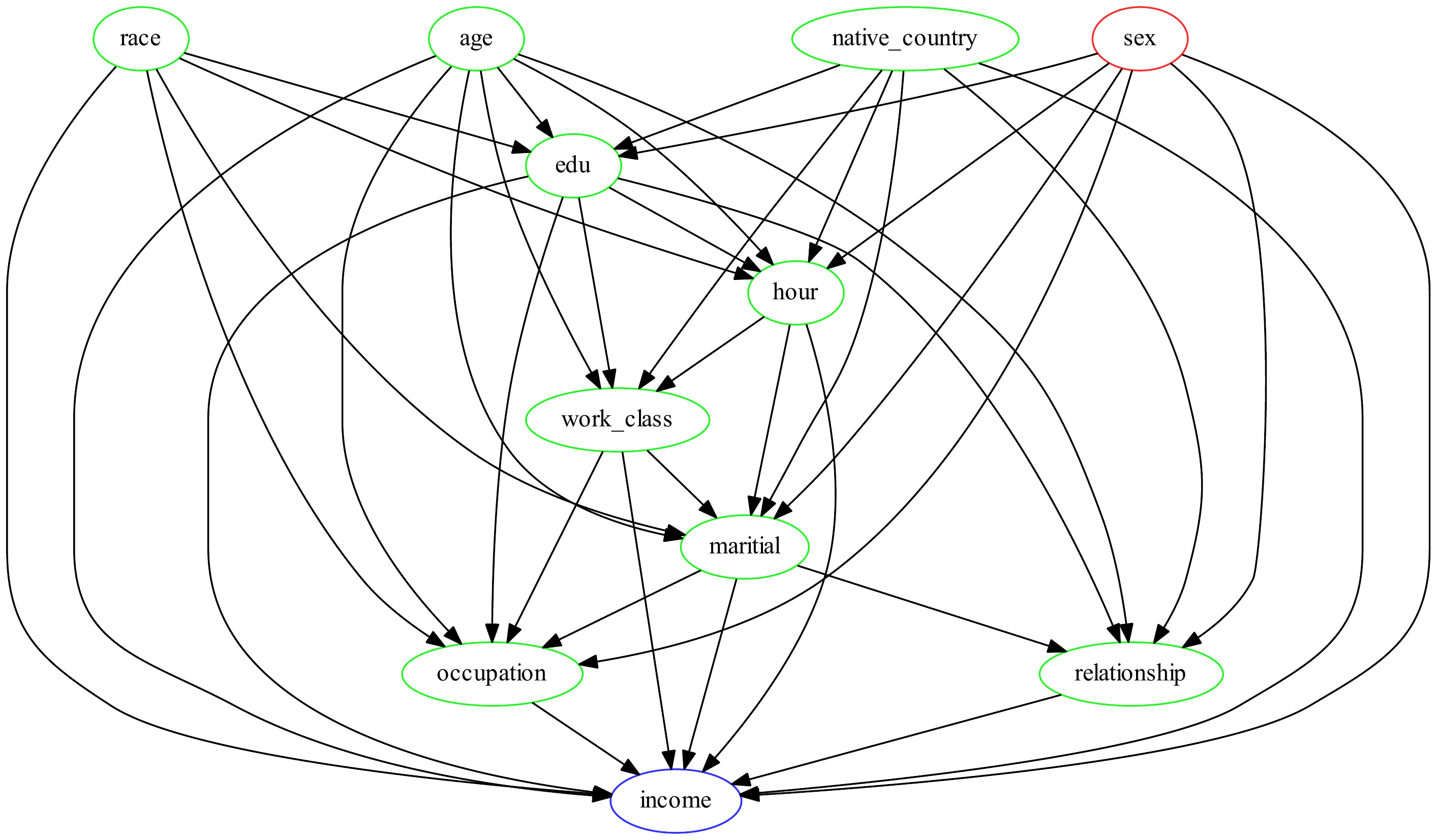}
	\caption{Causal graph for Adult dataset: the red node represents the protected attribute, the blue node represents the decision, the green nodes represent set $\mathbf{Q}$.}
	\label{fig:adult}
\end{figure*}

\begin{figure*}[ht]
	\centering
		\includegraphics[width=6in]{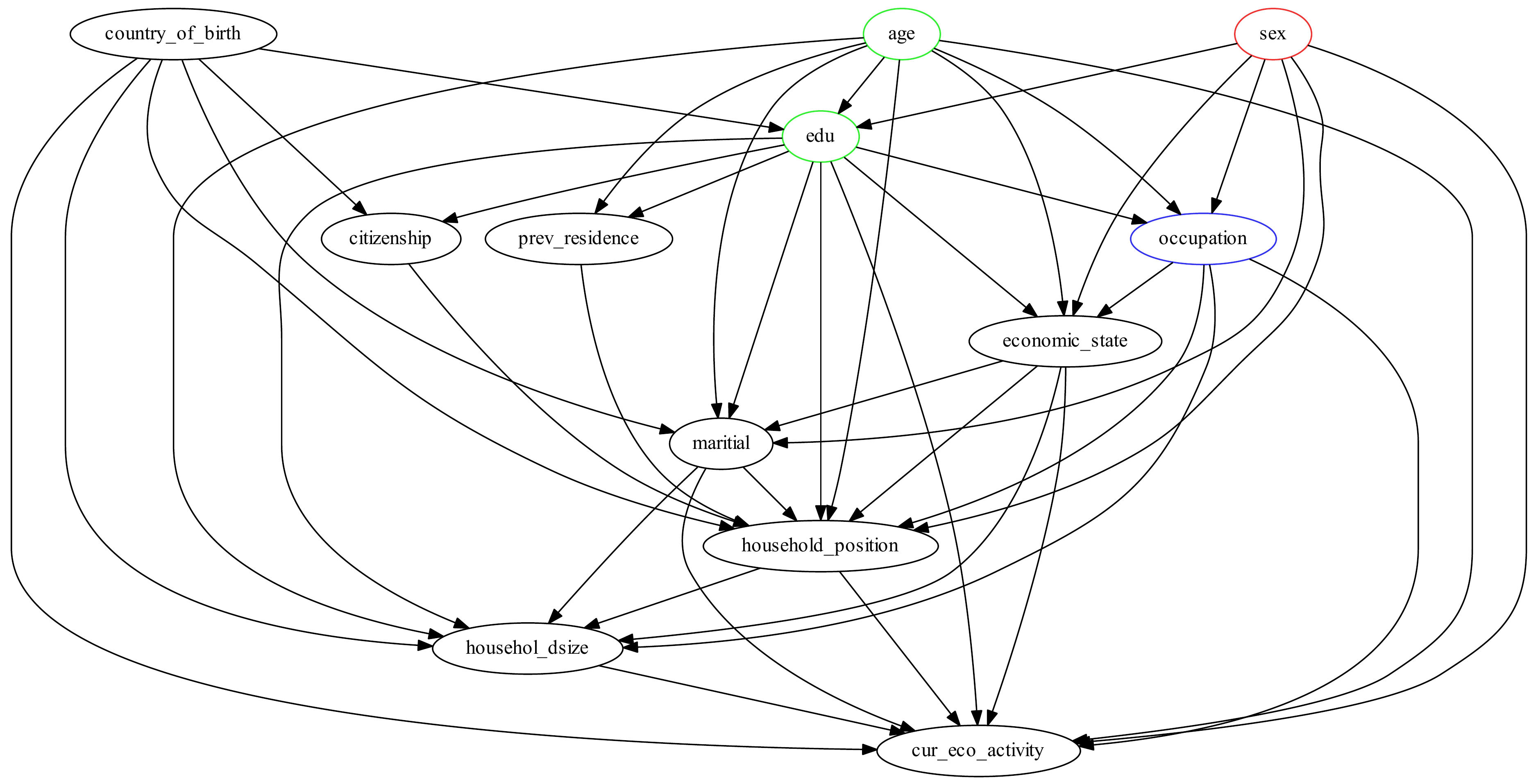}
	\caption{Causal graph for Dutch Census dataset: the red node represents the protected attribute, the blue node represents the decision, the green nodes represent set $\mathbf{Q}$, and the black nodes represent the others.}
	\label{fig:dutch}
\end{figure*}

The Adult dataset consists of $65123$ tuples with $11$ attributes such as \texttt{age}, \texttt{eduation}, \texttt{sex}, \texttt{occupation}, \texttt{income}, etc.. Since the computational complexity of the PC algorithm is an exponential function of the number of attributes and their domain sizes, for computational feasibility we binarize each attribute's domain values into two classes to reduce the domain sizes. 
We use three tiers in the partial order for temporal priority: \texttt{sex}, \texttt{age}, \texttt{native\_country}, \texttt{race} are defined in the first tier, \texttt{education} is defined in the second tier, and all other attributes are defined in the third tier. The constructed causal graph is shown in Figure \ref{fig:adult}.
We treat \texttt{sex} (female and male) as the protected attribute and \texttt{income} (low\_income and high\_income) as the decision. An arc pointing from \texttt{sex} to \texttt{income} is observed.
We first find set $\mathbf{Q}$ of \texttt{income}, which contains all the non-protected attributes. There are 512 subpopulations specified by $\mathbf{Q}$, and 376 subpopulations with non-zero number of tuples. Then, we compute $\Delta P|_{\mathbf{q}}$ for the 376 subpopulations. 
The value of $\Delta P|_{\mathbf{q}}$ ranges from $-0.85$ to $0.67$ across all subpopulations. Among them, 90 subpopulations have $\Delta P|_{\mathbf{q}} > 0.05$ and 49 subpopulations have $\Delta P|_{\mathbf{q}} < -0.05$, indicating the existence of discrimination in the Adult dataset.
Moreover, the mean and the standard variance of $\Delta P|_{\mathbf{q}}$ are $0.004$ and $0.129$, which 
 has small $\Pr(|\Delta P|_{\mathbf{q}}| < \tau )$ based on the Chebyshev's inequality, e.g., $\Pr(|\Delta P|_{\mathbf{q}}| < 0.15)\geq 25.97\%$. It 
 indicates that the non-discrimination cannot be claimed for the Adult dataset even under the relaxed $\alpha$-non-discrimination model with large $\tau$ and small $\alpha$.


Another dataset Dutch census consists of $60421$ tuples with $12$ attributes. Similarly, we binarize the domain values of attribute \texttt{age} due to its large domain size. Three tiers are used in the partial order for temporal priority: \texttt{sex}, \texttt{age}, \texttt{country\_birth} are defined in the first tire, \texttt{education\_level} is defined in the second tire, and all other attributes are defined in the third tire. The constructed causal graph is shown in Figure \ref{fig:dutch}. We treat \texttt{sex} (female and male) as the protected attribute and \texttt{occupation} (occupation\_w\_low\_income, occupation\_w\_high\_income) as the decision. An arc from \texttt{sex} to \texttt{occupation} is observed in the causal graph. Set $\mathbf{Q}$ of \texttt{occupation} is $\mathbf{Q}=\{\texttt{edu\_level},\texttt{age}\}$. The value of $\Delta P|_{\mathbf{q}}$ ranges from $0.062$ to $0.435$ across all the 12 subpopulations specified by $\mathbf{Q}$. Thus, discrimination against females is detected in the Dutch dataset based on the non-discrimination criterion. Moreover, the mean and the standard variance of $\Delta P|_{\mathbf{q}}$ are $0.222$ and $0.125$, which has small $\Pr(|\Delta P|_{\mathbf{q}}| < \tau )$ based on the Chebyshev's inequality, e.g., $\Pr(|\Delta P|_{\mathbf{q}}| < 0.30) \geq 27.94\%$. Hence, the Dutch census dataset still contains discrimination based on the relaxed $\alpha$-non-discrimination criterion.

Our current implementation uses the PC algorithm to construct the complete causal graph. 
In our experiment, the PC algorithm with the default significance threshold $0.01$ takes 51.59 seconds to build the graph for the binarized Adult dataset and 139.96 seconds for the binarized Dutch census dataset. We also run the PC algorithm on the original Adult dataset, which incurs 4492.36 seconds. In our future work, we will explore the use of the local causal discovery algorithms to improve the efficiency.

\subsection{Discrimination Removal}

\setlength{\tabcolsep}{4pt}

\begin{table}[htbp]
\small
	\centering
	\caption{Comparison of MGraph, MData, Naive, and two conditional discrimination removal algorithms (LM and LPS) on Adult and Dutch Census.} 
	\label{table:comparison}
	\begin{tabular}{|r|r|r||r|r|r|}
		\hline
		{\bf Adult} &    MGraph &  MData &   Naive &    LM &     LPS \\ \hline
		  $d(\times 10^{-3})$ & 	1.089&	1.28& 39.40& 50.99  & 35.22  \\ \hline	
		      $n_T$ &       676 &    790 &   29836 & 31506  &   14962 \\ \hline
		   $\chi^2$ &       210 &    422 & 19612 & 28943 & 11937  \\ \hline
		{\bf Dutch} &    MGraph &  MData &   Naive &    LM &     LPS \\ \hline
		  $d(\times 10^{-3})$ &  	5.09&	6.66& 13.64& 18.38 & 14.87 \\ \hline		
		      $n_T$ &      8776 &   8838 &  31908 &  30032  &   19998 \\ \hline
		   $\chi^2$ &      3478 &   8771 &  30990 &  30114  &  17209 \\ \hline
	\end{tabular}
\end{table}

The performance of our two proposed discrimination removal algorithms, MGraph and MData, in terms of the utility of the modified data is shown in Table \ref{table:comparison}.  We also report the results from the Naive method used in \cite{feldman2015certifying} in which we completely reshuffle the gender information. We measure the utility by three metrics: the Euclidean distance ($d$ ), the number of modified tuples ($n_T$), and the utility loss ($\chi^2$). We can observe from Table \ref{table:comparison} that the MGraph algorithm retains the highest utility. Both MGraph and MData algorithms significantly outperform the Naive method. We also examine how utility in terms of three metrics vary with different $\tau$ values for our MGraph and MData algorithms. 
We can see from Table~\ref{table:setting} that both discrimination removal algorithms incur less utility loss with larger $\tau$ values. This observation validates our analysis of non-discrimination model.

\setlength{\tabcolsep}{2.5pt}

\begin{table}[htb]
\small
\centering
\caption{Comparison of utility  with varied $\tau$ values for MGraph and MData.}
\begin{tabular}{|l|r|r|r|r|r|r|r|r|}
	\hline
	{\bf Adult} &  \multicolumn{4}{c|}{MGraph}   &    \multicolumn{4}{c|}{MData}     \\ \hline
	$\tau$      &  0.025 & 0.050 & 0.075 & 0.100 &  0.025 &  0.050 &  0.075 &  0.100 \\ \hline
	$d(\times 10^{-3})$   &   1.46&	1.08&	0.79&	0.56&	1.73&	1.28&	0.93&	0.68 \\ \hline	
	$n_T$       &   1046 &   676 &   476 &   326 &   1136 &    790 &    584 &    420 \\ \hline
	$\chi^2$    &    327 &   218 &   155 &   113 &    604 &    422 &    332 &    261 \\ \hline\hline
	{\bf Dutch} &  \multicolumn{4}{c|}{MGraph}   &    \multicolumn{4}{c|}{MData}     \\ \hline
	$\tau$      &  0.025 & 0.050 & 0.075 & 0.100 &  0.025 &  0.050 &  0.075 &  0.100 \\ \hline
	$d(\times 10^{-3})$   & 5.58&	5.09&	4.60&	4.14&	7.29&	6.66&	5.92&	5.28 \\ \hline
	$n_T$       &  10315 &  8776 &  7523 &  6595 &  10114 &   8838 &   7702 &   6800 \\ \hline
	$\chi^2$    &   4418 &  3478 &  2716 &  2132 &  11460 &   8771 &   6964 &   5658 \\ \hline
\end{tabular}
\label{table:setting}
\end{table}

We measure the execution times of our removal algorithms. As expected, MGraph takes longer time than MData since the former requires quadratic programming and data generation based on the whole modified graph while the latter only requires the information of $\mathbf{Q}$. For the Adult dataset with $\tau = 0.05$, MGraph takes 20.86s while MData takes 11.43s. For the Dutch dataset the difference is even larger, i.e., 735.83s for MGraph and 0.20s for MData, since the size of $\mathbf{Q}$ of Dutch census is much smaller.


\subsection{Comparison with conditional discrimination methods}
In \cite{zliobaite2011handling}, the authors measured the ``bad'' discrimination i.e., the effect that can be explained by conditioning on one attribute. They developed two methods, local massaging (LM) and local preferential sampling (LPS), to remove the unexplainable (bad) discrimination when one of the attributes is considered to be explanatory for the discrimination. However, their methods do not distinguish whether a partition is meaningful or not. Therefore, they cannot find the correct partitions to measure the direct discriminatory effects. Our experiments show that, their methods cannot completely remove discrimination conditioning on any single attribute. The results are skipped due to space limitation. In addition, even if we remove ``bad'' discrimination using their methods by conditioning on each attribute one by one, a significant amount of discriminatory effects still exist. After running the local massaging (LM) method, there are still 97 subpopultions (out of 376) with discrimination for Adult and 4 subpopulations (out of 12) with discrimination for Dutch census. The local preferential sampling (LPS) method performs even worse --- there are 108 subpopultions with discrimination for Adult and 8 subpopulations with discrimination for Dutch census. This is because for both datasets, any single attribute is not a block set and hence does not form a meaningful partition. Even assuming each attribute forms a meaningful partition, removing discrimination for each partition one by one does not guarantee to remove discrimination since the modification under one partition may change the distributions under other partitions. Differently, our approaches remove discrimination based on block set $\mathbf{Q}$ and ensure that the causal structure is not changed after the modification. Thus, Theorem \ref{thm:snd} can prove non-discrimination for our approaches. Furthermore, their methods incur much larger utility loss than our algorithms, as shown in the last two columns of Table \ref{table:comparison}.

\section{Related Work}\label{sec:rw}
A number of data mining techniques have been proposed to discover discrimination in the literature. Classification rule-based methods such as \emph{elift} \cite{pedreshi2008discrimination} and \emph{belift} \cite{mancuhan2014combating} were proposed to represent certain discrimination patterns. In \cite{luong2011k,zhang2016situation}, the authors dealt with the individual discrimination by finding a group of similar individuals. \v{Z}liobait\.{e} et al. \cite{zliobaite2011handling} proposed conditional discrimination. However, their approaches cannot determine whether a partition is meaningful and hence cannot achieve non-discrimination guarantee. Our work showed that the causal discriminatory effect through $C \rightarrow E$ can only be correctly measured under the partition specified by the block set. Recently, the authors in \cite{DBLP:journals/corr/BonchiHMR15} proposed a framework based on the Suppes-Bayes causal network and developed several random-walk-based methods to detect different types of discrimination. However, the construction of the Suppes-Bayes causal network is impractical with the large number of attribute-value pairs. In addition, it is unclear how the number of random walks is related to practical discrimination metrics, e.g., the difference in positive decision rates.

Proposed methods for discrimination removal are either based on data preprocessing \cite{kamiran2012data,zliobaite2011handling} or algorithm tweaking \cite{kamiran2010discrimination,calders2010three,kamishima2012fairness}. The authors \cite{dwork2012fairness} addressed the problem of fair classification that achieves both group fairness, i.e., the proportion of members in a protected group receiving positive classification is identical to the proportion in the population as a whole,  and individual fairness, i.e., similar individuals should be treated similarly. A recent work \cite{feldman2015certifying} studies how to remove disparate impact, i.e., indirect discrimination, from the data. The authors first ensures no direct discrimination by completely removing the protected attribute $C$ from the data. Then, they proposed to test disparate impact based on how well $C$ can be predicted with the non-protected attributes, and remove disparate impact by modifying the non-protected attributes. As shown by our experiments, removing $C$ from the data would significantly damage the data utility. Still, their work is based on correlations rather than the causation.

\section{Conclusions and Future Work}\label{sec:cc}
In this paper, we have investigated the problems of discovery and removal direct discrimination from historical decision data. With the support of the causal graph, we have shown that the discriminatory effect can only be identified under the partition defined by the block set. We have provided the graph condition for the block set. Based on that, we have developed a simple non-discrimination criterion and two strategies for removing discrimination. We also proposed a relaxed non-discrimination criterion to deal with sampling randomness in the data. The experiment results using real datasets show that our proposed approaches are effective in discovering and completely removing discrimination. Our work in this paper focuses on direct discrimination. We will investigate how to extend our work to modeling indirect discrimination using the causal graph and compare with the correlation-based indirect discrimination removal approach proposed in \cite{feldman2015certifying} in the future work.



\bibliographystyle{plain}
\bibliography{paper}

\appendix
\appendixpage

\section{Proof of Lemma \ref{thm:qd}}\label{sec:thm:qd}
\begin{proof}
We first sort the nodes in the causal graph according to the topological ordering of the DAG, so that for each sorted pair of nodes $X$ and $Y$ that $X$ is ahead of $Y$, $X$ must be $Y$'s non-descendent and $Y$ must be $X$'s non-ancestor. The topological ordering is guaranteed to be found in a DAG \cite{cormen2009introduction}. We represent the sorted nodes by an ordered list $\{\cdots,C,\cdots,E,\cdots\}$. According to the Markov condition, we have
\begin{equation}\label{eq:qd1}
\Pr(V|\mathrm{Prior}(V))=\Pr(V|\mathrm{Par}(V)),
\end{equation}
where $\mathrm{Prior}(V)$ represents all the nodes prior to $V$ in the ordering. Now we consider a topological ordering such that, (i) node $E$ and all nodes in $\mathbf{Q}$ are consecutive in ordering, (ii) all nodes posterior to $E$ are $E$'s descendents. It is easy to prove that such a topological ordering can always be constructed.\footnote{For (i), if any node lies between $E$ and some of its parents, we can move the node to the front of all $E$'s parents and the resultant list is still a topological ordering. Similarly we can prove (ii).} Denote by $\mathbf{X},\mathbf{Y},\mathbf{Z}$ the set of nodes that are prior to $C$, between $C$ and $\mathbf{Q}$, and posterior to $E$ respectively. The topological ordering can be represented as the list $\{\mathbf{X},C,\mathbf{Y},\mathbf{Q},E,\mathbf{Z}\}$. According to the definition of the block set, $\mathbf{B}$ contains no node in $\mathbf{Z}$. Thus, $\mathbf{B}\subseteq \mathbf{X}\cup \mathbf{Y}\cup \mathbf{Q}$. We define $\mathbf{X'}=\mathbf{X}\backslash \mathbf{B}$, $\mathbf{Y'}=\mathbf{Y}\backslash \mathbf{B}$, $\mathbf{Q'}=\mathbf{Q}\backslash \mathbf{B}$. Since sets $\mathbf{X},\mathbf{Y},\mathbf{Q}$ are mutually exclusive, we have $\mathbf{B}=(\mathbf{X}\backslash \mathbf{X'})\cup (\mathbf{Y}\backslash \mathbf{Y'})\cup (\mathbf{Q}\backslash \mathbf{Q'})$, which entails that
\begin{equation}\label{eq:qd2}
(\mathbf{X}\backslash \mathbf{X'})\cup (\mathbf{Y}\backslash \mathbf{Y'})=\mathbf{B}\cup \mathbf{Q'}.
\end{equation}
From probability theories, we have
\begin{equation*}
\begin{split}
& \Pr(e^{+}|c^{+},\mathbf{b})=\frac{\Pr(e^{+},c^{+},\mathbf{b})}{\Pr(c^{+},\mathbf{b})} \\
& =\frac{1}{\Pr(c^{+},\mathbf{b})} \sum_{\mathbf{X'},\mathbf{Y'},\mathbf{Q'},\mathbf{Z}}\Pr(\mathbf{x},c^{+},\mathbf{y},\mathbf{q},e^{+},\mathbf{z}).
\end{split}
\end{equation*}
According to the chain rule of probability calculus, we have
\begin{equation*}
\begin{split}
& \Pr(e^{+}|c^{+},\mathbf{b}) \\
& =\frac{1}{\Pr(c^{+},\mathbf{b})} \sum_{\mathbf{X'},\mathbf{Y'},\mathbf{Q'},\mathbf{Z}}\Pr(\mathbf{x},c^{+},\mathbf{y},\mathbf{q}) \\
& \qquad \cdot \Pr(e^{+}|\mathrm{Prior}(E))\cdot \Pr(\mathbf{z}|\mathrm{Prior}(\mathbf{Z})) \\
& =\frac{1}{\Pr(c^{+},\mathbf{b})} \sum_{\mathbf{X'},\mathbf{Y'},\mathbf{Q'}}\Pr(\mathbf{x},c^{+},\mathbf{y},\mathbf{q})\cdot \Pr(e^{+}|\mathrm{Prior}(E)).
\end{split}
\end{equation*}
From Equation \eqref{eq:qd1}, it follows that
\begin{equation*}
\begin{split}
& \Pr(e^{+}|c^{+},\mathbf{b}) \\
& =\frac{1}{\Pr(c^{+},\mathbf{b})} \sum_{\mathbf{X'},\mathbf{Y'},\mathbf{Q'}}\Pr(\mathbf{x},c^{+},\mathbf{y},\mathbf{q})\cdot \Pr(e^{+}|\mathrm{Par}(E)) \\
& =\frac{1}{\Pr(c^{+},\mathbf{b})} \sum_{\mathbf{X'},\mathbf{Y'},\mathbf{Q'}}\Pr(\mathbf{x},c^{+},\mathbf{y},\mathbf{q})\cdot \Pr(e^{+}|c^{+},\mathbf{q}) \\
& =\frac{1}{\Pr(c^{+},\mathbf{b})} \sum_{\mathbf{Q'}} \Big\{ \Pr(e^{+}|c^{+},\mathbf{q})\cdot \sum_{\mathbf{X'},\mathbf{Y'}}\Pr(\mathbf{x},c^{+},\mathbf{y},\mathbf{q}) \Big\}.
\end{split}
\end{equation*}
From Equation \eqref{eq:qd2}, we have
\begin{equation*}
\begin{split}
& \Pr(e^{+}|c^{+},\mathbf{b}) =\frac{1}{\Pr(c^{+},\mathbf{b})} \sum_{\mathbf{Q'}}\Pr(e^{+}|c^{+},\mathbf{q})\cdot \Pr(c^{+},\mathbf{b},\mathbf{q}) \\
& = \sum_{\mathbf{Q'}}\Pr(\mathbf{q}|c^{+},\mathbf{b})\cdot \Pr(e^{+}|c^{+},\mathbf{q}) \\
& = \sum_{\mathbf{Q'}}\Pr(\mathbf{q'}|c^{+},\mathbf{b})\cdot \Pr(e^{+}|c^{+},\mathbf{q}).
\end{split}
\end{equation*}

If $(C \Perp \mathbf{Q'} \mid \mathbf{B})_{\mathcal{G}'}$, then we can find a path from $C$ to $E$ through $\mathbf{Q}'$ that is not blocked, which means that $(C \Perp E \mid \mathbf{B})_{\mathcal{G}'}$. This contradicts $\mathbf{B}$ being a block set. Therefore, we must have $(C \Perp \mathbf{Q'} \mid \mathbf{B})_{\mathcal{G}'}$, which entails $(C \Perp \mathbf{Q'} \mid \mathbf{B})_{\mathcal{G}}$ according to the $d$-separation criterion. Thus, it follows that
\begin{equation*}
\Pr(e^{+}|c^{+},\mathbf{b}) = \sum_{\mathbf{Q'}}\Pr(\mathbf{q'}|\mathbf{b})\cdot \Pr(e^{+}|c^{+},\mathbf{q}).
\end{equation*}

We can obtain similar result for $\Pr(e^{+}|c^{-},\mathbf{b})$. Therefore, we have
\begin{equation}\label{eq:qd3}
\Delta P|_{\mathbf{b}}=\sum_{\mathbf{Q'}}\Pr(\mathbf{q'}|\mathbf{b})\cdot \Delta P|_{\mathbf{q}}.
\end{equation}

Hence, the lemma is proven.
\end{proof}

\section{Proof of Lemma \ref{thm:q}}\label{sec:thm:q}
\begin{proof}
We classify the paths from $C$ to $E$ other than arc $C \rightarrow E$ into two cases based on the last node $X$ ahead of $E$ on the path. For the first case, $X$ is a parent of $E$. Thus, $X$ is a noncollider and belongs to $\mathbf{Q}$. Based on the definition, each path in the first case is blocked by $\mathbf{Q}$. For the second case, $X$ is a child of $E$. Then, there must be at least one collider $Y$ on each path in the second case. Otherwise, the path is mono-directional with all the arcs pointing from $E$ to $C$, forming a circle with the arc $C \rightarrow E$. This contradicts to that a CBN is a directed acyclic graph. Let $Y$ be the last collider ahead of $E$ on a path. Then, neither $Y$ nor its descendant $Z$ can be $E$'s parent. Otherwise, mono-directional path $E \rightarrow \cdots \rightarrow Y \rightarrow E$ or $E \rightarrow \cdots \rightarrow Y \rightarrow \cdots \rightarrow Z \rightarrow E$ forms a circle, which again contradicts to that a CBN is a directed acyclic graph. Thus, according to the definition, each path in the second case is blocked by $\mathbf{Q}$. Finally, $\mathbf{Q}$ contains none of $E$'s descendents. Therefore, $\mathbf{Q}$ is a block set.

Hence, the lemma is proven.
\end{proof}

\section{Proof of Lemma \ref{thm:mu}}\label{sec:thm:mu}
\begin{proof}
By definition, we have
\begin{equation*}
\hat{\mu}_{\mathbf{B}}=\sum_{\mathbf{B}}\Pr(\mathbf{b})\cdot \Delta P|_{\mathbf{b}}.
\end{equation*}
According to Equation \eqref{eq:qd3}, we have
\begin{equation*}
\hat{\mu}_{\mathbf{B}}=\sum_{\mathbf{B}}\Pr(\mathbf{b})\cdot \sum_{\mathbf{Q'}}\Pr(\mathbf{q'}|\mathbf{b})\cdot \Delta P|_{\mathbf{q}},
\end{equation*}
where $\mathbf{Q'}=\mathbf{Q}\backslash \mathbf{B}$. It follows that
\begin{equation*}
\begin{split}
& \hat{\mu}_{\mathbf{B}}=\sum_{\mathbf{B},\mathbf{Q'}}\Pr(\mathbf{b})\cdot \Pr(\mathbf{q'}|\mathbf{b})\cdot \Delta P|_{\mathbf{q}} = \sum_{\mathbf{B},\mathbf{Q'}}\Pr(\mathbf{b},\mathbf{q'})\cdot \Delta P|_{\mathbf{q}} \\
& = \sum_{\mathbf{X}=\mathbf{B\cup Q'}}\Pr(\mathbf{x})\cdot \Delta P|_{\mathbf{q}} = \sum_{\mathbf{B'},\mathbf{Q}}\Pr(\mathbf{b'},\mathbf{q})\cdot \Delta P|_{\mathbf{q}},
\end{split}
\end{equation*}
where $\mathbf{B'}=\mathbf{B}\backslash \mathbf{Q}$. Then, it follows that
\begin{equation*}
\hat{\mu}_{\mathbf{B}}=\sum_{\mathbf{Q}}\Delta P|_{\mathbf{q}}\cdot \sum_{\mathbf{B'}}\Pr(\mathbf{b'},\mathbf{q}) = \sum_{\mathbf{Q}}\Delta P|_{\mathbf{q}}\cdot \Pr(\mathbf{q}) = \mu_{\mathbf{Q}}.
\end{equation*}

Hence, the lemma is proven.
\end{proof}

\section{Proof of Lemma \ref{thm:sigma}}\label{sec:thm:sigma}
\begin{proof}
By definition, we have
\begin{equation*}
\begin{split}
& \hat{\sigma}_{\mathbf{B}}^{2}=\sum_{\mathbf{B}}\Pr(\mathbf{b})(\Delta P|_{\mathbf{b}}-\hat{\mu}_{\mathbf{B}})^{2} \\
& = \sum_{\mathbf{B}}\Pr(\mathbf{b})\Big( (\Delta P|_{\mathbf{b}})^{2} - 2\hat{\mu}_{\mathbf{B}}\Delta P|_{\mathbf{b}} + \hat{\mu}_{\mathbf{B}}^{2} \Big) \\
& = \sum_{\mathbf{B}}\Pr(\mathbf{b})(\Delta P|_{\mathbf{b}})^{2} -2\hat{\mu}_{\mathbf{B}}\sum_{\mathbf{B}}\Pr(\mathbf{b})\Delta P|_{\mathbf{b}}  + \hat{\mu}_{\mathbf{B}}^{2}\sum_{\mathbf{B}}\Pr(\mathbf{b}).
\end{split}
\end{equation*}

According to Equation \eqref{eq:qd3}, we have
\begin{equation*}
\begin{split}
& \sum_{\mathbf{B}}\Pr(\mathbf{b})\cdot (\Delta P|_{\mathbf{b}})^{2} = \sum_{\mathbf{B}}\Pr(\mathbf{b})\cdot \Big(\sum_{\mathbf{Q'}}\Pr(\mathbf{q'}|\mathbf{b})\cdot \Delta P|_{\mathbf{q}}\Big)^{2} \\
& = \sum_{\mathbf{B}}\Pr(\mathbf{b})\cdot \Big(\sum_{\mathbf{Q'}}\sqrt{\Pr(\mathbf{q'}|\mathbf{b})}\cdot \sqrt{\Pr(\mathbf{q'}|\mathbf{b})} \Delta P|_{\mathbf{q}}\Big)^{2}.
\end{split}
\end{equation*}
According to Cauchy's Inequality, it follows that
\begin{equation*}
\begin{split}
& \sum_{\mathbf{B}}\Pr(\mathbf{b})\cdot (\Delta P|_{\mathbf{b}})^{2} \\
& \leq \sum_{\mathbf{B}}\Pr(\mathbf{b})\cdot \Big(\sum_{\mathbf{Q'}}\Pr(\mathbf{q'}|\mathbf{b})\Big) \cdot \Big(\sum_{\mathbf{Q'}}\Pr(\mathbf{q'}|\mathbf{b})\cdot (\Delta P|_{\mathbf{q}})^{2}\Big) \\
& = \sum_{\mathbf{B}}\Pr(\mathbf{b})\cdot \Big(\sum_{\mathbf{Q'}}\Pr(\mathbf{q'}|\mathbf{b})\cdot (\Delta P|_{\mathbf{q}})^{2}\Big).
\end{split}
\end{equation*}
Similar to the proof of Lemma 4.1, it follows that
\begin{equation*}
\begin{split}
& \sum_{\mathbf{B}}\Pr(\mathbf{b})\cdot (\Delta P|_{\mathbf{b}})^{2} \leq \sum_{\mathbf{B},\mathbf{Q'}}\Pr(\mathbf{b},\mathbf{q'})\cdot (\Delta P|_{\mathbf{q}})^{2} \\
& = \sum_{\mathbf{X}=\mathbf{B\cup Q'}}\Pr(\mathbf{x})\cdot (\Delta P|_{\mathbf{q}})^{2} = \sum_{\mathbf{B'},\mathbf{Q}}\Pr(\mathbf{b'},\mathbf{q})\cdot (\Delta P|_{\mathbf{q}})^{2} \\
& = \sum_{\mathbf{Q}}\Pr(\mathbf{q})\cdot (\Delta P|_{\mathbf{q}})^{2}.
\end{split}
\end{equation*}
Hence,  we have
\begin{equation*}
\sum_{\mathbf{B}}\Pr(\mathbf{b})\cdot (\Delta P|_{\mathbf{b}})^{2} \leq \sum_{\mathbf{Q}}\Pr(\mathbf{q})\cdot (\Delta P|_{\mathbf{q}})^{2}.
\end{equation*}
According to Lemma 4.1, we have
\begin{equation*}
\hat{\mu}_{\mathbf{B}}=\sum_{\mathbf{B}}\Pr(\mathbf{b})\cdot \Delta P|_{\mathbf{b}}=\hat{\mu}_{\mathbf{Q}}=\sum_{\mathbf{Q}}\Pr(\mathbf{q})\cdot \Delta P|_{\mathbf{q}}.
\end{equation*}
Besides, we have
\begin{equation*}
\sum_{\mathbf{B}}\Pr(\mathbf{b})=\sum_{\mathbf{Q}}\Pr(\mathbf{q})=1.
\end{equation*}
Thus, it follows that
\begin{equation*}
\begin{split}
& \hat{\sigma}_{\mathbf{B}}^{2} \leq \sum_{\mathbf{Q}}\Pr(\mathbf{q})(\Delta P|_{\mathbf{q}})^{2} -2\hat{\mu}_{\mathbf{Q}}\sum_{\mathbf{Q}}\Pr(\mathbf{q})\Delta P|_{\mathbf{q}} + \hat{\mu}_{\mathbf{Q}}^{2}\sum_{\mathbf{Q}}\Pr(\mathbf{q})\\
& =\hat{\sigma}_{\mathbf{Q}}^{2}.
\end{split}
\end{equation*}

Hence, the lemma is proven.
\end{proof}

\end{document}